\documentclass{article}
\pdfoutput=1
\usepackage{times}
\usepackage{graphicx} 
\usepackage{subfigure} 

\usepackage[usenames,dvipsnames]{color}

\usepackage{natbib}

\usepackage{algorithm}
\usepackage{algorithmic}

\usepackage[accepted]{icml-like}

\usepackage[utf8]{inputenc}
\usepackage{booktabs}
\usepackage{amsthm}
\usepackage{amsmath}  

\usepackage[tiny,compact]{titlesec}

\usepackage{ownstyles}

\usepackage{hyperref}

\newtheorem{proposition}{Proposition}

\newtheorem{corollary}{Corollary}
\newtheorem{theorem}{Theorem}

\newcommand{\svs}[1]{}
\newcommand{\vs}[1]{}

\newcommand{\otherf}{f}
\newcommand{\calP}{P}

\icmltitlerunning{Deep Generative Stochastic Networks Trainable by Backprop}

\begin{document} 

\twocolumn[
\icmltitle{Deep Generative Stochastic Networks Trainable by Backprop}



\icmlauthor{Yoshua Bengio$^*$}{find.us@on.the.web}
\icmlauthor{\'Eric Thibodeau-Laufer}{}
\icmlauthor{Guillaume Alain}{}
\icmladdress{D\'epartement d'informatique et recherche op\'erationnelle, Universit\'e de Montr\'eal,$^*$\& Canadian Inst. for Advanced Research}
\icmlauthor{Jason Yosinski}{}
\icmladdress{Department of Computer Science, Cornell University}

\icmlkeywords{deep learning, unsupervised learning, generative models, denoising autoencoder}

\vskip 0.3in
]

\begin{abstract}
We introduce a novel training principle for probabilistic models that is an alternative to maximum likelihood. The proposed Generative Stochastic Networks (GSN) framework is based on learning the transition operator of a Markov chain whose stationary distribution estimates the data distribution.  The transition distribution of the Markov chain is conditional on the previous state, generally involving a small move, so this conditional distribution has fewer dominant modes, being unimodal in the limit of small moves. Thus, it is easier to learn because it is easier to approximate its partition function, more like learning to perform supervised function approximation, with gradients that can be obtained by backprop. We provide theorems that generalize recent work on the probabilistic interpretation of denoising autoencoders and obtain along the way an interesting justification for dependency networks and generalized pseudolikelihood, along with a definition of an appropriate joint distribution and sampling mechanism even when the conditionals are not consistent. GSNs can be used with missing inputs and can be used to sample subsets of variables given the rest.  We validate these theoretical results with experiments on two image datasets using an architecture that mimics the Deep Boltzmann Machine Gibbs sampler but allows training to proceed with simple backprop, without the need for layerwise pretraining.
\end{abstract}

\begin{figure}[ht!]
\vs{3}
\centering
\includegraphics[width=1\linewidth]{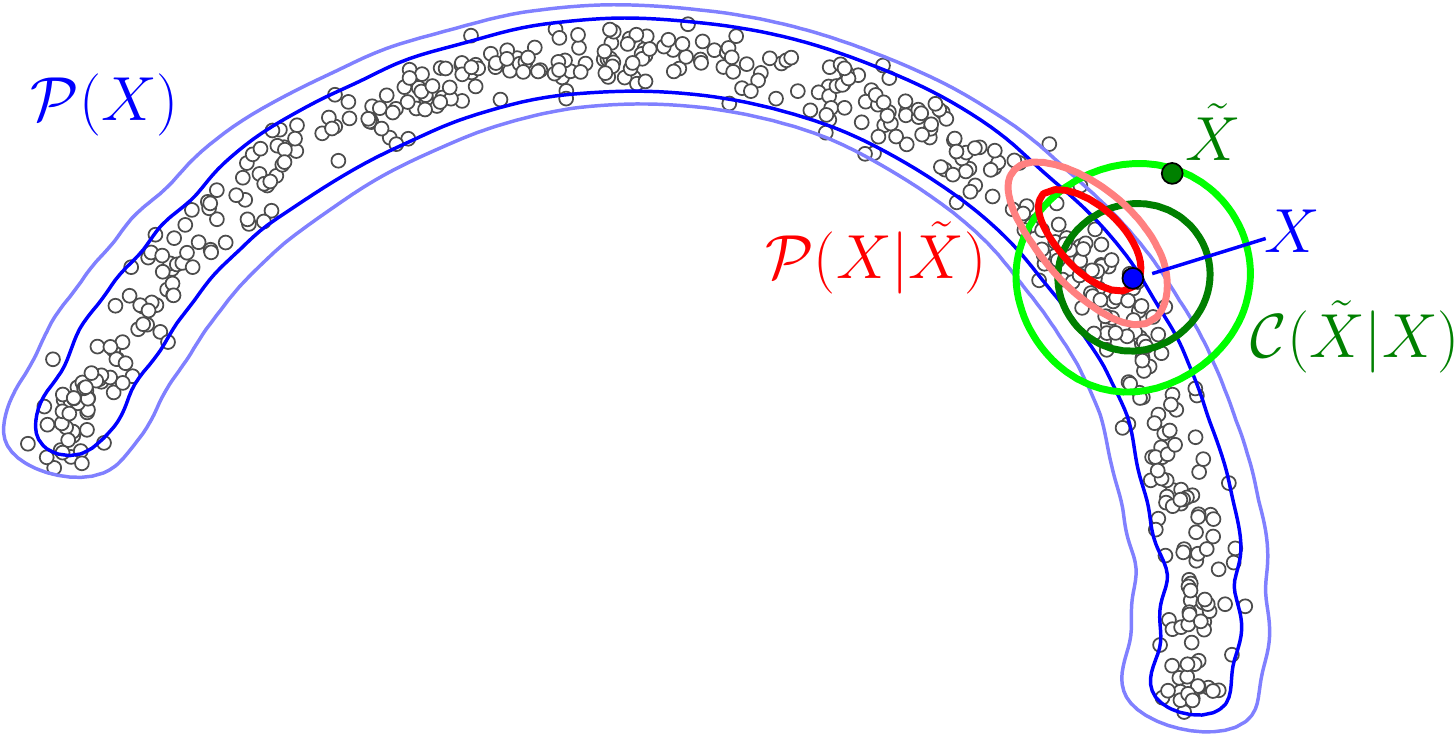}
\includegraphics[width=1\linewidth]{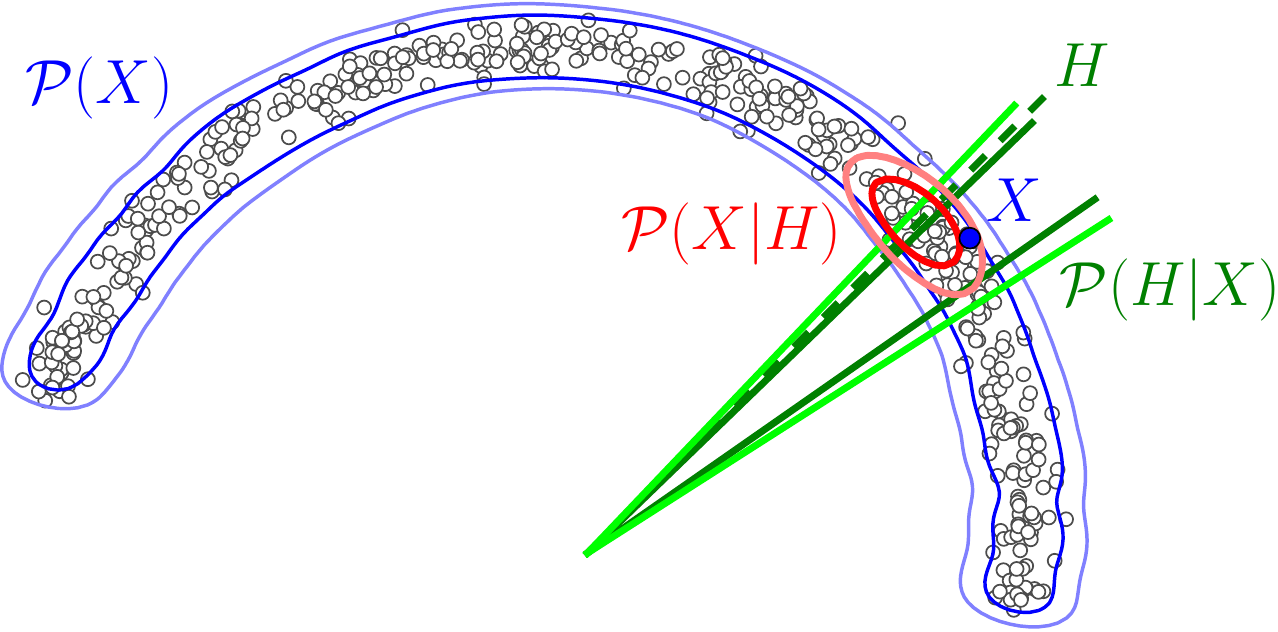}
\vs{3}
\caption{
{\em Top:} A denoising auto-encoder defines an estimated Markov chain where the transition 
operator first samples a corrupted $\tilde{X}$ from ${\cal C}(\tilde{X}|X)$ and then samples a reconstruction from $P_\theta(X|\tilde{X})$, which is trained to estimate the ground truth 
${\calP}(X|\tilde{X})$. Note how for any given $\tilde{X}$,
${\calP}(X|\tilde{X})$ is a much simpler (roughly unimodal) distribution than the
ground truth ${\calP}(X)$ and its partition function is thus easier to approximate.
{\em Bottom:} More generally, a GSN allows the use of arbitrary 
latent variables $H$ in addition to $X$, with the Markov chain state 
(and mixing) involving both $X$ and $H$. Here $H$ is the angle about the origin.
The GSN inherits the benefit of a simpler conditional and adds 
latent variables, which allow far more powerful deep representations in which mixing is easier~\citep{Bengio-et-al-ICML2013}.
}
\label{fig:data_px}
\vs{3}
\end{figure}

\svs{3}
\section{Introduction}
\svs{2}

Research in deep learning~(see \citet{Bengio-2009-book}
and~\citet{Bengio-Courville-Vincent-TPAMI2013} for reviews) grew from
breakthroughs in unsupervised learning of representations, based
mostly on the Restricted Boltzmann Machine (RBM)~\citep{Hinton06},
auto-encoder
variants~\citep{Bengio-nips-2006-small,VincentPLarochelleH2008-small}, and
sparse coding variants~\citep{HonglakLee-2007,ranzato-07-small}.
However, the most impressive recent results have been obtained with purely
supervised learning techniques for deep networks, in particular for speech
recognition~{\citep{dahl2010phonerec-small,Deng-2010,Seide2011}
and object recognition~\citep{Krizhevsky-2012-small}.
The latest breakthrough in object
recognition~\citep{Krizhevsky-2012-small} was achieved with fairly deep
convolutional networks with a form of noise injection in the input and hidden
layers during training,
called dropout~\citep{Hinton-et-al-arxiv2012}.
In all of these cases, the availability of large quantities of labeled data was critical.

On the other hand, progress with deep unsupervised architectures has been
slower, with the established options with a probabilistic footing
being the Deep Belief Network (DBN)~\citep{Hinton06} and the Deep Boltzmann
Machine (DBM)~\citep{Salakhutdinov+Hinton-2009-small}. 
Although single-layer unsupervised learners are fairly well
developed and used to pre-train these deep models, 
jointly training all the layers with respect to a single
unsupervised criterion remains a challenge, with a few
techniques arising to reduce that difficulty~\citep{Montavon2012,Goodfellow-et-al-NIPS2013}.
In contrast to recent progress toward joint supervised training of models with many layers, joint 
unsupervised training of deep models remains a difficult task.

Though the goal of training large unsupervised networks has turned out to
be more elusive than its supervised counterpart, the vastly larger
available volume of unlabeled data still beckons for efficient methods to
model it.  Recent progress in training supervised models raises the question:
can we take advantage of this progress to improve our ability to train
deep, generative, unsupervised, semi-supervised or structured output models?

This paper lays theoretical foundations
for a move in this direction through the following main contributions:

{\bf  1 -- Intuition: } In Section~\ref{sec:unsup_hard} we discuss what we view as
basic motivation for studying alternate ways of training unsupervised probabilistic models, i.e., avoiding
the intractable sums or maximization involved in many approaches.

{\bf   2 -- Training Framework: } We generalize recent work on the generative
view of denoising autoencoders~\cite{Bengio-et-al-NIPS2013} by introducing
latent variables in the framework to define Generative Stochastic Networks
(GSNs) (Section~\ref{sec:gsn}). GSNs aim to estimate the data generating
distribution indirectly, by parametrizing the transition operator of a
Markov chain rather than directly parametrizing $P(X)$.
Most critically, {\em this framework transforms the unsupervised density estimation
problem into one which is more similar to supervised function approximation}.
This enables training by (possibly regularized) maximum likelihood
and gradient descent computed via simple back-propagation,
avoiding the need to compute intractable partition functions. Depending
on the model, this may allow us to draw from any number of recently demonstrated
supervised training tricks. 
For example, one could use a convolutional
architecture with max-pooling for parametric parsimony and computational
efficiency, 
or dropout ~\citep{Hinton-et-al-arxiv2012} to prevent co-adaptation of hidden
representations.

{\bf   3 -- General theory:}
Training the generative (decoding / denoising) component of a GSN 
$P(X|h)$ with noisy representation $h$ is often far easier than modeling $P(X)$ explicitly
(compare the blue and red distributions in \mbox{Figure~\ref{fig:data_px})}. We prove that if our estimated $P(X|h)$
is consistent (e.g. through maximum likelihood), then the
stationary distribution of the resulting chain is a consistent estimator of the
data generating density, ${\calP}(X)$ (Section~\ref{sec:GSN}).
We strengthen the consistency theorems introduced 
in~\citet{Bengio-et-al-NIPS2013} by showing that the corruption distribution may be purely local, not requiring support
over the whole domain of the visible variables (Section~\ref{sec:DAE}).

{\bf   4 -- Consequences of theory: } We show that the model is general and extends to a wide range of architectures,
including sampling procedures whose computation can be unrolled
as a Markov Chain, i.e., architectures that add noise during intermediate
computation in order to produce random samples of a desired distribution (Theorem~\ref{thm:noisy-reconstruction}).
An exciting frontier in machine learning is the problem
of modeling so-called structured outputs, i.e., modeling
a conditional distribution where the output is high-dimensional
and has a complex multimodal joint distribution (given the input variable). We show how GSNs can be used to support such structured output and missing values (Section~\ref{sec:missing_inputs}).

{\bf   5 -- Example application: } In Section~\ref{sec:gsn_experiment} we show an example application of the GSN theory to create a
deep GSN whose computational graph 
resembles the one followed by Gibbs sampling in deep Boltzmann machines (with continuous latent variables),
but that can be trained efficiently with back-propagated gradients
and without layerwise pretraining. 
Because the Markov Chain is defined over a state
$(X,h)$ that includes latent variables, we reap the dual advantage
of more powerful models for a given number of
parameters and better mixing in the chain as we add noise to variables
representing higher-level information, first suggested by the results
obtained by~\citet{Bengio-et-al-ICML2013}
and~\citet{Luo+al-AISTATS2013-small}. The experimental results show that
such a model with latent states indeed mixes better
than shallower models without them (Table~\ref{tab:LL}).

{\bf   6 -- Dependency networks: } Finally, an unexpected result falls out of the GSN theory:
it allows us to provide a novel justification for dependency networks~\citep{HeckermanD2000}
and for the first time define a proper joint distribution between all the visible
variables that is learned by such models (Section~\ref{sec:dependency-nets}).

\svs{2}
\section{Summing over too many major modes}
\label{sec:unsup_hard}
\svs{2}


Many of the computations involved in graphical models (inference, sampling,
and learning) are made intractable and difficult to approximate because
of the large number of non-negligible modes in the modeled distribution
(either directly $P(x)$ or a joint distribution $P(x,h)$ involving latent variables $h$).
In all of these cases, what is intractable is the computation or approximation
of a sum (often weighted by probabilities), such as a marginalization or the estimation
of the gradient of the normalization constant. If only a few terms in this
sum dominate (corresponding to the dominant modes of the distribution), then
many good approximate methods can be found, such as Monte-Carlo Markov
chains (MCMC) methods.

Similarly difficult tasks arise with
structured output problems where one wants to sample from $P(y,h|x)$
and both $y$ and $h$ are high-dimensional and have a complex highly multimodal
joint distribution (given $x$).

Deep Boltzmann machines~\citep{Salakhutdinov+Hinton-2009-small} combine the difficulty of
inference (for the {\em positive phase} where one tries to push the
energies associated with the observed $x$ down) and also that of
sampling (for the {\em negative phase} where one tries to push up the
energies associated with $x$'s sampled from $P(x)$).  
Unfortunately, using an MCMC method to sample from $P(x,h)$ in order to
estimate the gradient of the partition function may be seriously hurt by
the presence of a large number of important modes, as argued below.



To evade the problem of highly multimodal joint or posterior
distributions, the currently known
approaches to dealing with the above intractable sums make very strong explicit 
assumptions (in the parametrization) or
implicit assumptions (by the choice of approximation methods) on the form of the distribution of interest.
In particular, MCMC methods are more likely to produce a good estimator
if the number of non-negligible modes is small: otherwise the
chains would require at least as many MCMC steps as the number of such
important modes, times a factor that accounts for the mixing time
between modes. Mixing time itself can be  very problematic as a trained model
becomes sharper, as it approaches a data generating distribution
that may have well-separated and sharp modes (i.e., manifolds).

We propose to make another assumption that might suffice to bypass
this multimodality problem: the
effectiveness of function approximation. 

In particular, the GSN approach presented in the next section relies on estimating the transition operator
of a Markov chain, e.g. $P(x_t | x_{t-1})$ or $P(x_t, h_t | x_{t-1}, h_{t-1})$.
Because each step of the Markov chain is generally local, these transition
distributions will often include only a very small number of important modes
(those in the neighbourhood of the previous state). Hence the gradient of their
partition function will be easy to approximate. For example consider the
denoising transitions studied by~\citet{Bengio-et-al-NIPS2013} and illustrated
in Figure~\ref{fig:data_px},
where $\tilde{x}_{t-1}$ is a stochastically corrupted version of $x_{t-1}$
and we learn the denoising distribution $P(x | \tilde{x})$.
In the extreme case (studied empirically here) where $P(x | \tilde{x})$
is approximated by a unimodal distribution, the only form of training
that is required involves function approximation (predicting the clean $x$
from the corrupted $\tilde{x}$).

Although having the true $P(x | \tilde{x})$ turn out to be unimodal
makes it easier to find an appropriate family of models for it,
unimodality is by no means required by the GSN framework itself.  One
may construct a GSN using any multimodal model for output (e.g. mixture
of Gaussians, RBMs, NADE, etc.), provided that gradients for the parameters of
the model in question can be estimated (e.g. log-likelihood gradients).



The approach proposed here thus avoids the need for a poor approximation
of the gradient of the partition function in the inner loop of training,
but still has the potential of capturing very rich distributions by relying
mostly  on ``function approximation''.

Besides the approach discussed here, there may well be other
very different ways of evading this problem of 
intractable marginalization, including
approaches such as sum-product
networks~\citep{Poon+Domingos-2011b}, which are based on learning a probability function that
has a tractable form by construction and yet is from a flexible enough family
of distributions.

\svs{2}
\section{Generative Stochastic Networks}
\label{sec:gsn}
\svs{2}


Assume the problem we face is to construct a model for some unknown
data-generating distribution ${\calP}(X)$ given only examples of $X$ drawn
from that distribution. In many cases, the unknown distribution
${\calP}(X)$ is complicated, and modeling it directly can be difficult.

A recently proposed approach using denoising autoencoders transforms the 
difficult task of modeling ${\calP}(X)$ into a supervised learning problem that may be much easier to solve. The basic approach is as follows: given a clean example data point $X$ from ${\calP}(X)$, we obtain a corrupted version $\tilde{X}$ by sampling from some corruption distribution ${\cal C}(\tilde{X}|X)$. For example, we might take a clean image, $X$, and add random white noise to produce $\tilde{X}$. We then use supervised learning methods to train a function to reconstruct, as accurately as possible, any $X$ from the data set given only a noisy version $\tilde{X}$. As shown in Figure~\ref{fig:data_px}, the reconstruction distribution ${\calP}(X|\tilde{X})$ may often be much easier to learn than the data distribution ${\calP}(X)$, {\em because ${\calP}(X|\tilde{X})$
tends to be dominated by a single or few major modes} (such as the roughly Gaussian shaped
density in the figure).

But how does learning the reconstruction distribution help us solve our
original problem of modeling ${\calP}(X)$? The two problems are clearly
related, because if we knew everything about ${\calP}(X)$, then our
knowledge of the ${\cal C}(\tilde{X}|X)$ that we chose would allow us to
precisely specify the optimal reconstruction function via Bayes rule:
${\calP}(X|\tilde{X}) = \frac{1}{z} {\cal C}(\tilde{X}|X){\calP}(X)$,
where $z$ is a normalizing constant that does not depend on $X$. As one
might hope, the relation is also true in the opposite direction: once we
pick a method of adding noise, ${\cal C}(\tilde{X}|X)$, knowledge of the
corresponding reconstruction distribution ${\calP}(X|\tilde{X})$ is
sufficient to recover the density of the data ${\calP}(X)$.

This intuition was borne out by proofs in two recent papers.
\citet{Alain+Bengio-ICLR2013} showed that denoising auto-encoders
with small Gaussian corruption and squared error loss estimated the
score (derivative of the log-density with respect to the input) of 
continuous observed random variables. More recently, \citet{Bengio-et-al-NIPS2013}
generalized this to arbitrary variables (discrete, continuous or both),
arbitrary corruption (not necessarily asymptotically small), and arbitrary
loss function (so long as they can be seen as a log-likelihood).

Beyond proving that ${\calP}(X|\tilde{X})$ is sufficient to reconstruct the data density, \citet{Bengio-et-al-NIPS2013} also demonstrated a method of sampling from a learned, parametrized model of the density, $P_\theta(X)$, by running a Markov chain that alternately adds noise using ${\cal C}(\tilde{X}|X)$ and denoises by sampling from the learned $P_\theta(X|\tilde{X})$, which is trained to approximate the true ${\calP}(X|\tilde{X})$. The most important contribution of that paper was demonstrating that if a learned, parametrized reconstruction function $P_\theta(X|\tilde{X})$ converges to the true ${\calP}(X|\tilde{X})$, then under some relatively benign conditions the stationary distribution $\pi(X)$ of the resulting Markov chain will exist and will indeed converge to the data distribution ${\calP}(X)$.

Before moving on, we should pause to make an important point clear. Alert readers may have noticed that ${\calP}(X|\tilde{X})$ and ${\calP}(X)$ can each be used to reconstruct the other given knowledge of ${\cal C}(\tilde{X}|X)$. Further, if we assume that we have chosen a simple ${\cal C}(\tilde{X}|X)$ (say, a uniform Gaussian with a single width parameter), then ${\calP}(X|\tilde{X})$ and ${\calP}(X)$ must both be of approximately the same complexity. Put another way, we can never hope to combine a simple ${\cal C}(\tilde{X}|X)$ and a simple ${\calP}(X|\tilde{X})$ to model a complex ${\calP}(X)$. Nonetheless, it may still be the case that ${\calP}(X|\tilde{X})$ is easier to {\em model} than ${\calP}(X)$
 due to reduced computational
complexity in computing or approximating the partition functions of the conditional distribution mapping corrupted input $\tilde{X}$ to 
the distribution of corresponding clean input $X$. Indeed, because that conditional is going to be mostly assigning
probability to $X$ locally around $\tilde{X}$, ${\calP}(X|\tilde{X})$ has only one or a few modes, while ${\calP}(X)$ can have
a very large number.

So where did the complexity go? ${\calP}(X|\tilde{X})$ has fewer modes than ${\calP}(X)$, but {\em the location of these modes depends on the value of $\tilde{X}$}. It is precisely this mapping from $\tilde{X} \rightarrow$ {\em mode location} that allows us to trade a difficult density modeling problem for a supervised function approximation problem that admits application of many of the usual supervised learning tricks.


In the next four sections, we extend previous results in several directions.

\svs{2}
\subsection{Generative denoising autoencoders with local noise}
\label{sec:DAE}
\svs{2}

The main theorem in \citet{Bengio-et-al-NIPS2013}, reproduced below, requires that the
Markov chain be ergodic.
A set of conditions
guaranteeing ergodicity is given in the aforementioned paper, but
these conditions are restrictive in requiring that ${\cal C}(\tilde{X}|X)>0$ everywhere
that ${\calP}(X)>0$.
Here we show how to relax these conditions and still
guarantee ergodicity through other means.

Let $P_{\theta_n}(X | \tilde{X})$
be a denoising auto-encoder that has been trained on $n$ training examples.
$P_{\theta_n}(X | \tilde{X})$
assigns a probability to $X$, given $\tilde{X}$,
when $\tilde{X} \sim {\cal C}(\tilde{X}|X)$.
This estimator defines a Markov chain $T_n$ obtained by sampling
alternatively an $\tilde{X}$ from ${\cal C}(\tilde{X}|X)$ and
an $X$ from $P_\theta(X | \tilde{X})$. Let $\pi_n$ be the asymptotic
distribution of the chain defined by $T_n$, if it exists.
The following theorem is proven by~\citet{Bengio-et-al-NIPS2013}.
\begin{theorem}
\label{thm:consistency}
{\bf If}
$P_{\theta_n}(X | \tilde{X})$ is a consistent estimator of the true conditional
distribution ${\calP}(X | \tilde{X})$
{\bf and}
$T_n$ defines an 
ergodic Markov chain,
{\bf then}
as $n\rightarrow \infty$, the asymptotic distribution $\pi_n(X)$ of the generated
samples converges to the data-generating distribution ${\calP}(X)$. 
\end{theorem}

In order for Theorem~\ref{thm:consistency} to apply, the chain must be ergodic. One set of conditions under which this occurs is given in the aforementioned paper. We slightly restate them here:

\begin{figure}[htpb]
\centering
\includegraphics[width=1\linewidth]{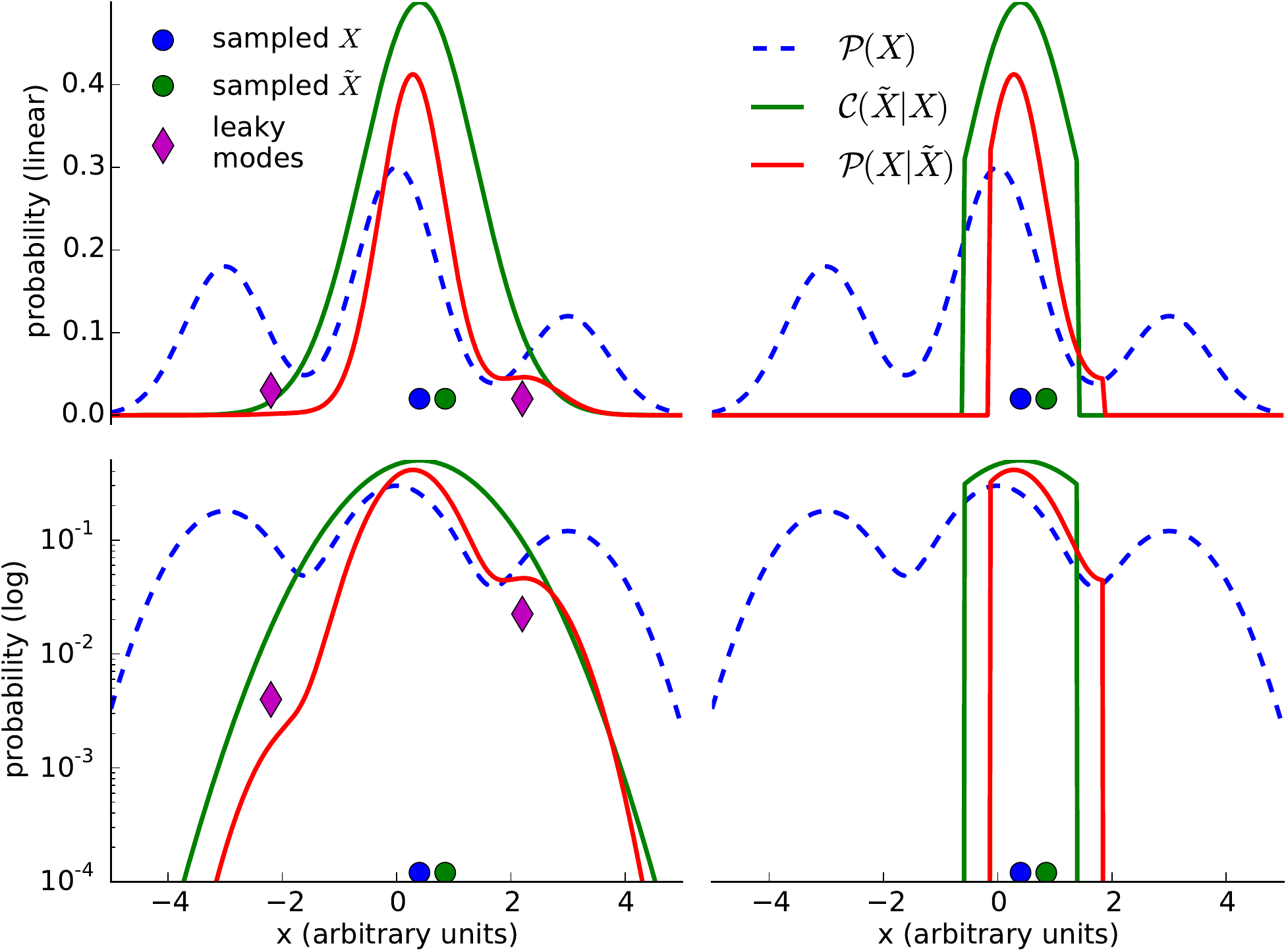}
\caption{If ${\cal C}(\tilde{X}|X)$ is globally supported as required by Corollary~\ref{cor:ergonew} \citep{Bengio-et-al-NIPS2013}, then for $P_{\theta_n}(X|\tilde{X})$ to converge to ${\calP}(X|\tilde{X})$, it will eventually have to model all of the modes in ${\calP}(X)$, even though the modes are damped (see ``leaky modes'' on the left). However, if we guarantee ergodicity through other means, as in Corollary~\ref{cor:ergonew2}, we can choose a local ${\cal C}(\tilde{X}|X)$ and allow $P_{\theta_n}(X|\tilde{X})$ to model only the local structure of ${\calP}(X)$ (see right).}
\label{fig:modes_1d}
\end{figure}

\begin{corollary}
\label{cor:ergonew}
{\bf If} the support for both the data-generating distribution and
denoising model are contained in and non-zero in
a finite-volume region $V$ (i.e., $\forall \tilde{X}$, $\forall X\notin V,\; {\calP}(X)=0, P_\theta(X|\tilde{X})=0$ and $\forall \tilde{X}$, $\forall X\in V,\; {\calP}(X)>0, P_\theta(X|\tilde{X})>0,  {\cal C}(\tilde{X}|X)>0$)
{\bf and} these statements remain
true in the limit of $n\rightarrow \infty$, {\bf then}
the chain defined by $T_n$ will be ergodic.
\end{corollary}

If conditions in Corollary~\ref{cor:ergonew} apply, then the chain will be ergodic and Theorem~\ref{thm:consistency} will apply. However, these conditions are sufficient, not necessary, and in many cases they may be artificially restrictive. In particular, Corollary~\ref{cor:ergonew} 
defines a large region $V$ containing any possible $X$ allowed by the model and requires that we maintain the probability of jumping between any two points in a single move to be greater than 0.
While this generous condition helps us easily guarantee the ergodicity of the chain, it also has the unfortunate side effect of requiring that, in order for $P_{\theta_n}(X|\tilde{X})$ to converge to the conditional distribution ${\calP}(X|\tilde{X})$, it must have the capacity to model every mode of ${\calP}(X)$, exactly the difficulty we were trying to avoid. The left two plots in Figure~\ref{fig:modes_1d} show this difficulty: because ${\cal C}(\tilde{X}|X)>0$ everywhere in $V$, every mode of $P(X)$ will leak, perhaps attenuated, into $P(X|\tilde{X})$.

Fortunately, we may seek ergodicity through other means. The following
corollary allows 
us to choose a ${\cal C}(\tilde{X}|X)$ that only makes small jumps, which in turn only requires $P_\theta(X|\tilde{X})$ to model a small part of the space $V$ around each $\tilde{X}$.

Let $P_{\theta_n}(X | \tilde{X})$
be a denoising auto-encoder that has been trained on $n$ training examples
and ${\cal C}(\tilde{X}|X)$ be some corruption distribution.
$P_{\theta_n}(X | \tilde{X})$
assigns a probability to $X$, given $\tilde{X}$,
when $\tilde{X} \sim {\cal C}(\tilde{X}|X)$ and $X \sim {\cal P}(X)$.
Define a Markov chain $T_n$ by alternately sampling
an $\tilde{X}$ from ${\cal C}(\tilde{X}|X)$ and
an $X$ from $P_\theta(X | \tilde{X})$.

\begin{corollary}
\label{cor:ergonew2}
{\bf If} the data-generating distribution is contained in and non-zero in
a finite-volume region $V$ (i.e., $\forall X\notin V,\; {\calP}(X)=0,$ and $\forall X\in V,\; {\calP}(X)>0$)
{\bf and} all pairs of points in $V$ can be connected by a finite-length path through $V$
{\bf and}
for some $\epsilon > 0$,
$\forall \tilde{X} \in V,\forall X\in V$ within $\epsilon$ of each other,
${\cal C}(\tilde{X}|X) > 0$ and $P_{\theta}(X|\tilde{X}) > 0$
{\bf and} these statements remain
true in the limit of $n\rightarrow \infty$,
{\bf then}
the chain defined by $T_n$ will be ergodic.
\end{corollary}

\begin{proof}
Consider any two points $X_a$ and $X_b$ in $V$.
By the assumptions of Corollary~\ref{cor:ergonew2}, there exists a finite length path between $X_a$ and $X_b$ through $V$. Pick one such finite length path $P$.
Chose a finite series of points $x = \{x_1, x_2, \ldots, x_k\}$ along $P$, with $x_1 = X_a$ and $x_k = X_b$ such that the distance between every pair of consecutive points $(x_i, x_{i+1})$ is less than $\epsilon$ as defined in Corollary~\ref{cor:ergonew2}.
Then the probability of sampling $\tilde{X} = x_{i+1}$ from ${\cal C}(\tilde{X}|x_i))$ will be positive, because ${\cal C}(\tilde{X}|X)) > 0$ for all $\tilde{X}$ within $\epsilon$ of $X$ by the assumptions of Corollary~\ref{cor:ergonew2}. Further, the probability of sampling $X = \tilde{X} = x_{i+1}$ from $P_{\theta}(X|\tilde{X})$ will be positive from the same assumption on $P$.
Thus the probability of jumping along the path from $x_i$ to $x_{i+1}$, $T_n(X_{t+1} = x_{i+1}|X_{t} = x_i)$, will be greater than zero for all jumps on the path.
Because there is a positive probability finite length path between all pairs of points in $V$, all states commute, and the chain is irreducible.
If we consider $X_a = X_b \in V$, by the same arguments $T_n(X_t = X_a|X_{t-1} = X_a) > 0$. Because there is a positive probability of remaining in the same state, the chain will be aperiodic.
Because the chain is irreducible and over a finite state space, it will be positive recurrent as well. Thus, the chain defined by $T_n$ is ergodic.
\end{proof}

Although this is a weaker condition that has the advantage of making
the denoising distribution even easier to model (probably having less
modes), we must be careful to choose the ball size $\epsilon$ large
enough to guarantee that one can jump often enough between the major modes of ${\calP}(X)$
when these are separated by zones of tiny probability. $\epsilon$ must be larger than half the 
largest distance one would have to travel across a desert of low probability
separating two nearby modes (which if not connected in this way would make
$V$ not anymore have a single connected component). Practically, there would
be a trade-off between the difficulty of estimating ${\calP}(X|\tilde{X})$
and the ease of mixing between major modes separated by a very low density zone.

The generalization of the above results presented in the next section is meant to help
deal with this mixing problem. It is inspired by the recent work~\citep{Bengio-et-al-ICML2013}
showing that mixing between modes can be a serious problem for RBMs and DBNs,
and that well-trained deeper models can greatly alleviate it by allowing the
mixing to happen at a more abstract level of representation (e.g., where some
bits can actually represent which mode / class / manifold is considered).

\svs{2}
\subsection{Generalizing the denoising autoencoder to GSNs}
\label{sec:GSN}
\svs{2}


The denoising auto-encoder Markov chain is defined by 
$\tilde{X}_t \sim C(\tilde{X}|X_t)$ and 
$X_{t+1} \sim P_\theta(X | \tilde{X}_t)$, where $X_t$ alone can
serve as the state of the chain. The GSN framework generalizes this by defining
a Markov chain with both a visible $X_t$ and a latent variable $H_t$ as state
variables, of the form
\vspace{-0.4em}
\begin{eqnarray*}
   H_{t+1} &\sim& P_{\theta_1}(H | H_t, X_t) \\
   X_{t+1} &\sim& P_{\theta_2}(X | H_{t+1}).
\end{eqnarray*}
\begin{center}
\includegraphics[width=0.4\textwidth]{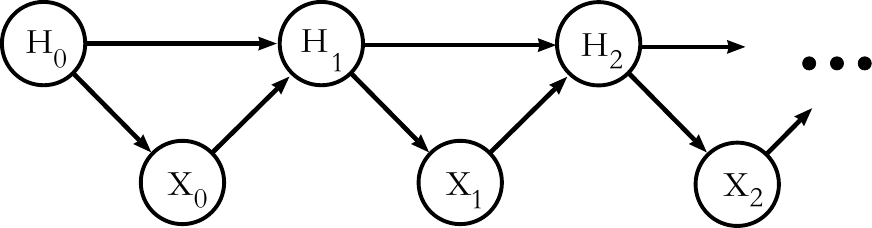} 
\end{center}
Denoising auto-encoders are thus a special case of GSNs.
Note that, given that the distribution of $H_{t+1}$ depends on a previous value of $H_t$,
we find ourselves with an extra $H_0$ variable added at the beginning of the chain.
This $H_0$ complicates things when it comes to training, but when
we are in a sampling regime we can simply wait a sufficient
number of steps to burn in.

The next theoretical results give conditions for making the stationary distributions
of the above Markov chain match a target data generating distribution.


\begin{theorem}
\label{thm:noisy-reconstruction}
Let ${(H_t,X_t)}_{t=0}^\infty$ be the Markov chain defined by the following graphical model.
\begin{center}
\includegraphics[width=0.4\textwidth]{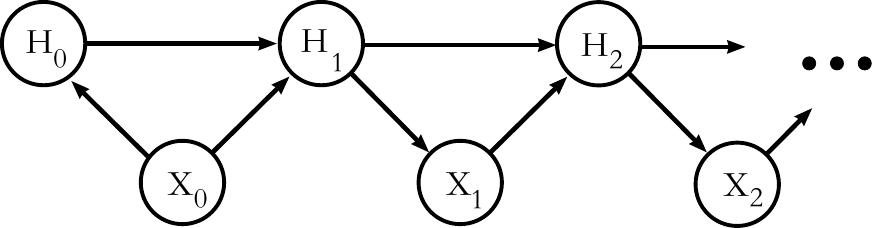} 
\end{center}
\vspace{-0.4em}
If we assume that the chain has a stationary distribution $\pi_{X,H}$, and that for every value of $(x,h)$ we have that
\vspace{-0.4em}
\begin{itemize}
\item all the $P(X_t = x | H_t=h) = g(x,h)$ share the same density for $t \geq 1$
\item all the $P(H_{t+1} = h | H_t=h', X_t=x) = f(h, h', x)$ share the same density for $t \geq 0$
\item $P(H_0=h|X_0=x)=P(H_1=h|X_0=x)$
\item $P(X_1=x|H_1=h)=P(X_0=x|H_1=h)$
\end{itemize}
\vspace{-0.4em}
then for every value of $(x,h)$ we get that
\vspace{-0.6em}
\begin{itemize}
\item $P(X_0 = x | H_0=h) = g(x,h)$ holds, which is something that was assumed only for $t\geq 1$
\item $P(X_t = x, H_t = h) = P(X_0 = x, H_0 = h)$ for all $t\geq 0$
\item the stationary distribution $\pi_{H,X}$ has a marginal distribution $\pi_X$ such that $\pi\left(x\right) = P\left(X_0=x\right)$.
\end{itemize}
\vspace{-0.2em}
Those conclusions show that our Markov chain has the property that its samples in $X$ are drawn from the same distribution as $X_0$.
\end{theorem}

\begin{proof}
The proof hinges on a few manipulations done with the first variables to show
that $P(X_t = x | H_t=h) = g(x,h)$, which is assumed for $t \geq 1$, also holds for $t=0$.

For all $h$ we have that
\vspace{-0.1em}
\begin{eqnarray*}
P(H_0=h) & = & \int P(H_0=h|X_0=x)P(X_0=x) dx \\
         & = & \int P(H_1=h|X_0=x)P(X_0=x) dx \\
         & = & P(H_1=h).
\end{eqnarray*}

The equality in distribution between $(X_1,H_1)$ and $(X_0,H_0)$ is obtained with
\vspace{-0.1em}
\begin{eqnarray*}
P(X_1=x, H_1=h) & = & P(X_1=x | H_1=h) P(H_1=h) \\
                & = & P(X_0=x | H_1=h) P(H_1=h) \\
                &   & \hspace{1em} \textrm{(by hypothesis)} \\
                & = & P(X_0=x, H_1=h) \\
                & = & P(H_1=h | X_0=x)P(X_0=x) \\
                & = & P(H_0=h | X_0=x)P(X_0=x) \\
                &   & \hspace{1em} \textrm{(by hypothesis)} \\
                & = & P(X_0=x, H_0=h).
\end{eqnarray*}

Then we can use this to conclude that
\vspace{-0.1em}
\begin{eqnarray*}
         & P(X_0 = x, H_0 = h) = P(X_1 = x, H_1 = h) \\
\implies & P(X_0 = x |H_0 = h) = P(X_1 = x | H_1 = h) = g(x,h)
\end{eqnarray*}
so, despite the arrow in the graphical model being turned the other way, we have that
the density of $P(X_0 = x |H_0 = h)$ is the same as for all other $P(X_t = x |H_t = h)$ with $t \geq 1$.

Now, since the distribution of $H_1$ is the same as the distribution of $H_0$, and
the transition probability $P(H_1=h|H_0=h')$ is entirely defined by the $(f,g)$ densities
which are found at every step for all $t\geq 0$, then we know that $(X_2, H_2)$ will have the
same distribution as $(X_1, H_1)$. To make this point more explicitly,
\begin{align*}
   & P(H_1 = h | H_0 = h') \\
 = & \int P(H_1 = h | H_0 = h', X_0 = x)  P(X_0 = x | H_0 = h') dx \\
 = & \int f(h, h', x)  g(x, h') dx \\
 = & \int P(H_2 = h | H_1 = h', X_1 = x)  P(X_1 = x | H_1 = h') dx \\
 = & P(H_2 = h | H_1 = h')
\end{align*}
This also holds for $P(H_3|H_2)$ and for all subsequent $P(H_{t+1}|H_t)$.
This relies on the crucial step where we demonstrate that $P(X_0 = x | H_0 = h) = g(x,h)$.
Once this was shown, then we know that we are using the
same transitions expressed in terms of $(f,g)$ at every step.

Since the distribution of $H_0$ was shown above to be the same as the distribution of $H_1$,
this forms a recursive argument that shows that all the $H_t$ are equal in distribution to $H_0$.
Because $g(x,h)$ describes every $P(X_t = x | H_t = h)$, we have that
all the joints $(X_t, H_t)$ are equal in distribution to $(X_0, H_0)$.

This implies that the stationary distribution $\pi_{X,H}$ is the same as that of $(X_0, H_0)$.
Their marginals with respect to $X$ are thus the same.
\end{proof}

To apply Theorem \ref{thm:noisy-reconstruction} in a context
where we use experimental data to learn a model, we would like
to have certain guarantees concerning the robustness of
the stationary density $\pi_X$. When a model lacks capacity, or
when it has seen only a finite number of training examples,
that model can be viewed as a perturbed version of the exact
quantities found in the statement of
Theorem \ref{thm:noisy-reconstruction}.

A good overview of results from perturbation theory
discussing stationary distributions in finite state Markov chains
can be found in \citep{Cho2000comparisonperturbation}.
We reference here only one of those results.

\begin{theorem}\label{thm:schweitzer_inequality}
Adapted from \citep{Schweitzer1968perturbation}

Let $K$ be the transition matrix of a finite state, irreducible,
homogeneous Markov chain. Let $\pi$ be its stationary distribution
vector so that $K\pi=\pi$. Let $A=I-K$ and $Z=\left(A+C\right)^{-1}$
where $C$ is the square matrix whose columns all contain $\pi$.
Then, if $\tilde{K}$ is any transition matrix (that also satisfies
the irreducible and homogeneous conditions) with stationary distribution
$\tilde{\pi}$, we have that
\[
\left\Vert \pi-\tilde{\pi}\right\Vert _{1}\leq\left\Vert Z\right\Vert _{\infty}\left\Vert K-\tilde{K}\right\Vert _{\infty}.
\label{eqn:schweitzer_inequality}
\]

\end{theorem}

This theorem covers the case of discrete data by showing how the stationary
distribution is not disturbed by a great amount when the transition
probabilities that we learn are close to their correct values. We are
talking here about the transition between steps of the chain
$(X_0, H_0), (X_1, H_1), \ldots, (X_t, H_t)$, which are defined
in Theorem \ref{thm:noisy-reconstruction} through the $(f,g)$ densities.

\vspace{1cm}

We avoid discussing the training criterion for a GSN.
Various alternatives exist, but this analysis is for future work.
Right now Theorem \ref{thm:noisy-reconstruction} suggests the following
rules :
\begin{itemize}
\item Pick the transition distribution $f(h, h', x)$ to be useful.
There is no bad $f$ when $g$ can be trained perfectly with infinite capacity.
However, the choice of $f$ can put a great burden on $g$,
and using a simple f, such as one that represents additive gaussian noise,
will lead to less difficulties in training $g$.
In practice, we have also found good results by training $f$ at the
same time by back-propagating the errors from $g$ into $f$. In this way we
simultaneously train $g$ to model the distribution implied by $f$ and
train $f$ to make its implied distribution easy to model by $g$.
\item Make sure that during training
$P(H_0=h|X_0=x) \rightarrow P(H_1=h|X_0=x)$. One interesting way to achieve this is, for each $X_0$ in the training set, iteratively sample $H_1 | (H_0, X_0)$ and substitute the value of $H_1$ as the updated value of $H_0$. Repeat until you have achieved a kind of ``burn in''.
Note that, after the training is completed, when we use the chain for sampling, the samples that we get from its stationary distribution do not depend on $H_0$. This technique of substituting the $H_1$ into $H_0$ does not apply beyond the training step.
\item Define $g(x,h)$ to be your estimator for $P(X_0=x|H_1=h)$, e.g. by training an estimator of this conditional distribution from the samples $(X_0, H_1)$.
\item The rest of the chain for $t \geq 1$ is defined in terms of $(f,g)$.
\end{itemize}

As much as we would like to simply learn $g$ from pairs $(H_0, X_0)$,
the problem is that the training samples $X_0^{(i)}$ are descendants
of the corresponding values of $H_0^{(i)}$ in the original
graphical model that describes the GSN. Those $H_0^{(i)}$ are hidden quantities
in GSN and we have to find a way to deal with them. Setting them all to be
some default value would not work because the relationship between $H_0$ and
$X_0$ would not be the same as the relationship later between $H_t$ and $X_t$
in the chain.

\svs{2}
\subsection{Alternate parametrization with deterministic functions of random quantities}
\label{sec:alternate_params_deterministic_functions}
\svs{2}

There are several equivalent ways of
expressing a GSN. One of the interesting
formulations is to use deterministic functions
of random variables to express the densities $(f,g)$
used in Theorem \ref{thm:noisy-reconstruction}.
With that approach, we define $H_{t+1}=\otherf_{\theta_1}(X_t,Z_t,H_t)$
for some independent noise source $Z_t$,
and we insist that $X_t$ cannot be recovered
exactly from $H_{t+1}$. The advantage of that formulation
is that one can directly back-propagated the reconstruction
log-likelihood $\log P(X_1=x_0 | H_1=f(X_0,Z_0,H_0))$ into all the parameters
of $f$ and $g$ (a similar idea was independently proposed in~\citep{Kingma-arxiv2013} and
also exploited in~\citep{Rezende-et-al-arxiv2014}).

For the rest of this paper, we will use such
a deterministic function $\otherf$ instead of having
$f$ refer to a probability density function. We apologize if it
causes any confusion.

In the setting described at the beginning of section \ref{sec:gsn},
the function playing the role of the ``encoder'' was fixed for the purpose
of the theorem, and we showed that learning only the ``decoder'' part (but a
sufficiently expressive one) sufficed. In this
setting we are learning both, for which some care
is needed. 

One problem would be if the created Markov
chain failed to converge to a stationary distribution.
Another such problem could be that the function $\otherf(X_t,Z_t,H_t)$
learned would try to ignore the noise $Z_t$, or not make the best use
out of it. In that case, the reconstruction distribution would
simply converge to a Dirac at the input $X$.  This is
the analogue of the constraint on auto-encoders that is needed to prevent
them from learning the identity function. Here, we must design the family
from which $f$ and $g$ are learned 
such that when the noise $Z$ is injected, there are always
several possible values of $X$ that could have been the correct original
input.

Another extreme case to think about is when $\otherf(X,Z,H)$ is overwhelmed
by the noise and has lost all information about $X$. In that case the theorems
are still applicable while giving uninteresting results: the learner must
capture the full distribution of $X$ in $P_{\theta_2}(X|H)$ because
the latter is now equivalent to $P_{\theta_2}(X)$, since $\otherf(X,Z,H)$
no longer contains information about $X$. This illustrates
that when the noise is large, the reconstruction distribution (parametrized
by $\theta_2$) will need to have the expressive power to represent
multiple modes. Otherwise, the reconstruction will tend to capture 
an average output, which would visually look like a fuzzy combination
of actual modes. In the experiments performed here, we have only
considered unimodal reconstruction distributions (with factorized 
outputs), because we expect that even if ${\calP}(X|H)$
is not unimodal, it would be dominated by a single mode when the
noise level is small. However, 
future work should investigate multimodal alternatives.

A related element to keep in mind is that one should pick the
family of conditional distributions $P_{\theta_2}(X|H)$ so that
one can sample from them and one can easily train them when given
$(X,H)$ pairs, e.g., by maximum likelihood.

\svs{2}
\subsection{Handling missing inputs or structured output}
\label{sec:missing_inputs}
\svs{2}

In general, a simple way to deal with missing inputs is to clamp the observed
inputs and then apply the Markov chain with the constraint that the observed
inputs are fixed and not resampled at each time step, whereas the unobserved
inputs are resampled each time, {\em conditioned on the clamped inputs}.

One readily proves that this procedure gives
rise to sampling from the appropriate conditional distribution:

\begin{proposition}
\label{prop:clamping}
{\bf If}
a subset $x^{(s)}$ of the elements of $X$ is kept fixed (not resampled) while
the remainder $X^{(-s)}$ is updated stochastically 
during the Markov chain of Theorem~\ref{thm:noisy-reconstruction},
but using $P(X_t | H_t, X_t^{(s)}=x^{(s)})$,
{\bf then}
the asymptotic distribution $\pi_n$ of the Markov chain 
produces samples of $X^{(-s)}$ from
the conditional distribution $\pi_n(X^{(-s)}|X^{(s)}=x^{(s)})$.
\end{proposition}

\begin{proof}
Without constraint, we know that at convergence, the chain
produces samples of $\pi_n$.  A subset of these
samples satisfies the condition $X=x^{(s)}$, and these constrained samples
could equally have been produced by sampling $X_t$ from
\begin{equation*}
P_{\theta_2}(X_t | \otherf_{\theta_1}(X_{t-1},Z_{t-1},H_{t-1}), X_t^{(s)}=X^{(s)}),
\end{equation*}
by definition of conditional
distribution.  Therefore, at convergence of the chain, we have that using the constrained
distribution $P(X_t | \otherf(X_{t-1},Z_{t-1},H_{t-1}), X_t^{(s)}=x^{(s)})$ produces a sample from
$\pi_n$ under the condition $X^{(s)}=x^{(s)}$.
\end{proof}

Practically, it means that we must choose an output (reconstruction)
distribution from which it is not only easy to sample from, but
also from which it is easy to sample a subset of the variables in the vector $X$ {\em conditioned on the rest 
being known}. In the experiments below, we
used a factorial distribution for the reconstruction, from which
it is trivial to sample conditionally a subset of the
input variables. In general (with non-factorial output distributions)
one must use the proper conditional for the theorem to apply, i.e.,
it is not sufficient to clamp the inputs, one must also sample the
reconstructions from the appropriate conditional distribution (conditioning
on the clamped values).

This method of dealing with missing inputs can be immediately applied
to structured outputs. If $X^{(s)}$ is viewed as an ``input''
and $X^{(-s)}$ as an ``output'', then sampling from 
$X^{(-s)}_{t+1} \sim P(X^{(-s)} | \otherf((X^{(s)},X^{(-s)}_t), Z_t,H_t), X^{(s)})$
will converge to estimators of ${\calP}(X^{(-s)}|X^{(s)})$. This still
requires good choices of the parametrization (for $\otherf$ as well as for
the conditional probability $P$), but the advantages of this approach
are that there is no approximate inference of latent variables and
the learner is trained with respect to simpler conditional 
probabilities: in the limit of small noise, we conjecture that these 
conditional probabilities can be well approximated by unimodal distributions.
Theoretical evidence comes from~\citet{Alain+Bengio-ICLR2013}:
{\em when the amount of corruption noise converges to 0 and the input variables
have a smooth continuous density, then a unimodal Gaussian reconstruction
density suffices to fully capture the joint distribution.}

\begin{figure*}[ht]
\centering
\begin{center}
\includegraphics[width=\textwidth]{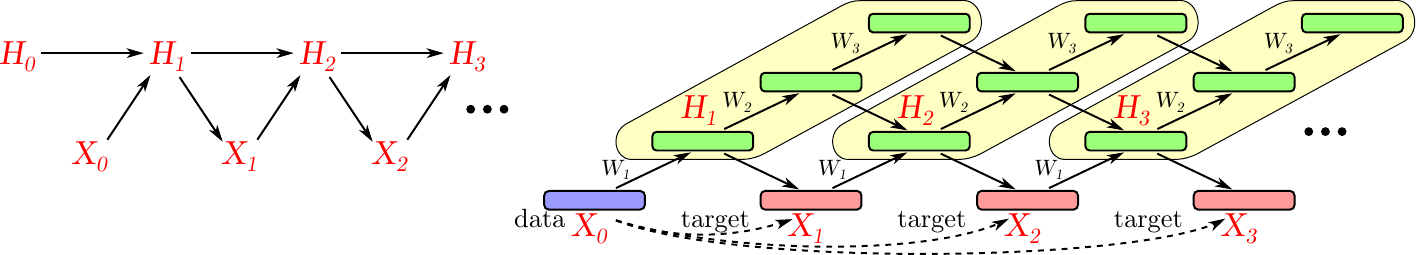}
\end{center}
\caption{{\em Left:} Generic GSN Markov chain with state variables
$X_t$ and $H_t$. {\em Right:} GSN Markov chain inspired by the
unfolded computational graph of the Deep Boltzmann
Machine Gibbs sampling process, but with backprop-able stochastic units
at each layer. The training example $X=x_0$ starts the chain. Either
odd or even layers are stochastically updated at each step, and all
downward weight matrices are fixed to the transpose of the corresponding
upward weight matrix. All $x_t$'s
are corrupted by salt-and-pepper noise before entering the graph.
Each $x_t$ for $t>0$ is obtained by sampling
from the reconstruction distribution for that step, $P_{\theta_2}(X_t|H_t)$. The walkback
training objective is the sum over all steps of log-likelihoods of target
$X=x_0$ under the reconstruction distribution.
In the special case of a unimodal Gaussian reconstruction
distribution, maximizing the likelihood is equivalent to minimizing
reconstruction error; in general one trains to maximum
likelihood, not simply minimum reconstruction error.}
\label{fig:comp-graph}
\vs{3}
\end{figure*}

\svs{2}
\subsection{Dependency Networks as GSNs}
\label{sec:dependency-nets}
\svs{2}

Dependency networks~ \mbox{\citep{HeckermanD2000}} are models in which one
estimates conditionals $P_i(x_i | x_{-i})$, where $x_{-i}$ denotes $x \setminus x_i$,
i.e., the set of variables other than the $i$-th one, $x_i$. Note that
each $P_i$ may be parametrized separately, thus not guaranteeing that
there exists a joint of which they are the conditionals. Instead of the
ordered pseudo-Gibbs sampler defined in~ \mbox{\citet{HeckermanD2000}}, which
resamples each variable $x_i$ in the order $x_1, x_2, \ldots$, we can view
dependency networks in the GSN framework by defining a proper Markov chain
in which at each step one randomly chooses which variable to resample. The
corruption process therefore just consists of $H=f(X,Z)=X_{-s}$ 
where $X_{-s}$ is the complement of $X_{s}$, with $s$
a randomly chosen subset of elements of $X$ (possibly constrained to be
of size 1).  Furthermore, we parametrize the reconstruction
distribution as $P_{\theta_2}(X=x|H) = \delta_{x_{-s}=X_{-s}}P_{\theta_2,s}(X_s=x_s | x_{-s})$ 
where the estimated conditionals
$P_{\theta_2,s}(X_s=x_s | x_{-s})$ are not constrained to be consistent
conditionals of some joint distribution over all of $X$.

\begin{proposition}
If the above GSN Markov chain has a stationary distribution, then
the dependency network defines a joint distribution (which is that
stationary distribution), which does not have to be known in closed
form. Furthermore, if the conditionals are consistent estimators
of the ground truth conditionals, then that stationary distribution 
is a consistent estimator of the ground truth joint distribution.
\vs{1}
\end{proposition}
The proposition can be proven by immediate application of Theorem 1
from ~\citet{Bengio-et-al-NIPS2013} 
with the above definitions of the GSN. This joint stationary
distribution can exist even if the conditionals are not consistent.
To show that, assume that some choice of (possibly inconsistent)
conditionals gives rise to a stationary distribution $\pi$.
Now let us consider the set of all conditionals (not necessarily
consistent) that could have given rise to that $\pi$.
Clearly, the conditionals derived from $\pi$ is part of that
set, but there are infinitely many others (a simple counting
argument shows that the fixed point equation of $\pi$ introduces
fewer constraints than the number of degrees of freedom that
define the conditionals). To better understand why the ordered
pseudo-Gibbs chain does not benefit from the same properties, we can consider an extended case
by adding an extra component of the state $X$, being
the index of the next variable to
resample. In that case, the Markov chain associated with the ordered pseudo-Gibbs procedure
would be periodic, thus violating the ergodicity assumption
of the theorem. However, by introducing randomness in the choice
of which variable(s) to resample next, we obtain aperiodicity
and ergodicity, yielding as stationary distribution
a mixture over all possible resampling orders. These results also
show in a novel way (see e.g. ~\citet{Hyvarinen-2006-small} for earlier results)
that training by pseudolikelihood or generalized pseudolikelihood
provides a consistent estimator of the associated joint, so long as the GSN
Markov chain defined above is ergodic. This result can
be applied to show that the multi-prediction deep Boltzmann machine (MP-DBM)
training procedure introduced by~\citet{Goodfellow-et-al-NIPS2013} also
corresponds to a GSN. This has been exploited in order to obtain
much better samples using the associated GSN Markov chain than
by sampling from the corresponding DBM~\citep{Goodfellow-et-al-NIPS2013}.
Another interesting conclusion that one can draw from this paper
and its GSN interpretation is that state-of-the-art classification
error can thereby be obtained: 0.91\% on MNIST without fine-tuning (best 
comparable previous DBM results was well above 1\%) and 10.6\% on permutation-invariant
NORB (best previous DBM results was 10.8\%). 

\svs{2}
\section{Experimental Example of GSN}
\label{sec:gsn_experiment}
\svs{2}

The theoretical results on Generative Stochastic Networks (GSNs)
open for exploration a large class of possible
parametrizations which will share the property that they can capture the
underlying data distribution through the GSN Markov chain. What
parametrizations will work well?  Where and how should one inject noise? We present
results of preliminary experiments with specific selections for each of these choices, but
the reader should keep in mind that the space of possibilities is vast.


As a conservative starting point, we propose to explore families of
parametrizations which are similar to existing deep stochastic architectures
such as the Deep Boltzmann Machine (DBM)~\citep{Salakhutdinov+Hinton-2009-small}.
Basically, the
idea is to construct a computational graph that is similar to the
computational graph for Gibbs sampling or variational inference
in Deep Boltzmann Machines. However, we have
to diverge a bit from these architectures in order to accommodate the
desirable property that it will be possible to back-propagate
the gradient of reconstruction log-likelihood with respect to the
parameters $\theta_1$ and $\theta_2$. Since the gradient of a binary
stochastic unit is 0 almost everywhere, we have to consider related
alternatives. An interesting source of inspiration regarding this
question is a recent paper on estimating or propagating gradients
through stochastic neurons~\citep{Bengio-arxiv2013}.
Here we consider the following stochastic non-linearities:
$h_i=\eta_{\rm out}+\tanh(\eta_{\rm in} + a_i)$
where $a_i$ is the linear activation for unit $i$ (an affine transformation
applied to the input of the unit, coming from the layer below, the layer
above, or both) and $\eta_{\rm in}$ and $\eta_{\rm out}$ are 
zero-mean Gaussian noises.

To emulate a sampling procedure similar to Boltzmann machines 
in which the filled-in missing values can depend on the
representations at the top level, the computational graph allows
information to propagate both upwards (from input to higher levels)
and downwards, giving rise to the
computational graph structure illustrated in Figure~\ref{fig:comp-graph},
which is similar to that explored for {\em deterministic} recurrent
auto-encoders~\citep{SeungS1998,Behnke-2001,Savard-master-small}. Downward
weight matrices have been fixed to the transpose of corresponding
upward weight matrices.

The {\em walkback} algorithm was proposed
in~\citet{Bengio-et-al-NIPS2013} to make training of generalized
denoising auto-encoders (a special case of the models studied here) more
efficient. The basic idea is that the reconstruction 
is obtained after not one but several steps of the sampling Markov chain. In this context
it simply means that the computational graph from $X$ to a reconstruction
probability actually involves generating intermediate samples as if we were
running the Markov chain starting at $X$. In the experiments, the
graph was unfolded so that $2 D$ sampled reconstructions would be produced,
where $D$ is the depth (number of hidden layers). The training loss is
the sum of the reconstruction negative log-likelihoods (of target $X$) 
over all those reconstruction steps.

Experiments evaluating the ability of the GSN models to generate good samples
were performed on the MNIST and TFD datasets, following the setup 
in~\citet{Bengio-et-al-ICML2013}.
Networks with 2 and 3 hidden layers
were evaluated and compared to regular denoising auto-encoders (just 1
hidden layer, i.e., the computational graph separates into separate ones
for each reconstruction step in the walkback algorithm). They all have tanh
hidden units and pre- and post-activation Gaussian noise of standard
deviation 2, applied to all hidden layers except the first.
In addition, at each step in the
chain, the input (or the resampled $X_t$) is corrupted with salt-and-pepper
noise of 40\% (i.e., 40\% of the pixels are corrupted, and replaced with a
0 or a 1 with probability 0.5). Training is over 100 to 600 epochs at most, with
good results obtained after around 100 epochs.  Hidden layer sizes vary
between 1000 and 1500 depending on the experiments, and a learning rate of
0.25 and momentum of 0.5 were selected to approximately minimize the
reconstruction negative log-likelihood. The learning rate is reduced
multiplicatively by $0.99$ after each epoch.  Following~\citet{Breuleux+Bengio-2011},
the quality of the samples
was also estimated quantitatively by measuring the log-likelihood
of the test set under a Parzen density estimator constructed from
10000 consecutively generated samples (using the real-valued mean-field reconstructions
as the training data for the Parzen density estimator). This can be
seen as an {\em lower bound on the true log-likelihood}, with the bound
converging to the true likelihood as we consider more samples and
appropriately set the smoothing parameter of the Parzen 
estimator\footnote{However, in this paper, to be consistent with
the numbers given in \citet{Bengio-et-al-ICML2013} we used a Gaussian
Parzen density, which makes the numbers not comparable with the
AIS log-likelihood upper bounds for binarized images reported in other papers
for the same data.}
Results are summarized in Table~\ref{tab:LL}. The test set Parzen log-likelihood bound
was not used to select among model architectures, but visual inspection of
samples generated did guide the preliminary search reported here.
Optimization hyper-parameters (learning rate, momentum, and
learning rate reduction schedule) were selected based on the 
reconstruction log-likelihood training objective. The Parzen log-likelihood bound obtained
with a two-layer model on MNIST is 214 ($\pm$ standard error of 1.1), while the log-likelihood
bound obtained by a single-layer model (regular denoising auto-encoder, DAE in
the table) is
substantially worse, at -152$\pm$2.2.
In comparison,~\citet{Bengio-et-al-ICML2013} report a log-likelihood bound of -244$\pm$54 for
RBMs and 138$\pm$2 for a 2-hidden layer DBN, using the same setup. We have
also evaluated a 3-hidden layer DBM~\citep{Salakhutdinov+Hinton-2009-small}, using
the weights provided by the author, and obtained a Parzen log-likelihood bound of 32$\pm$2.
See \texttt{http://www.mit.edu/{\textasciitilde}rsalakhu/DBM.html} for details.

\begin{figure}[htpb]
\vs{3}
\centering
\includegraphics[width=1.0\linewidth]{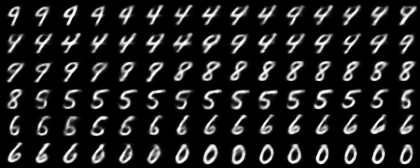} 

\vspace*{1mm}
\includegraphics[width=1.0\linewidth]{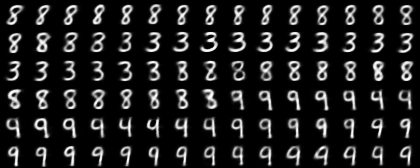}

\vspace*{1.5mm}
\includegraphics[width=1.0\linewidth]{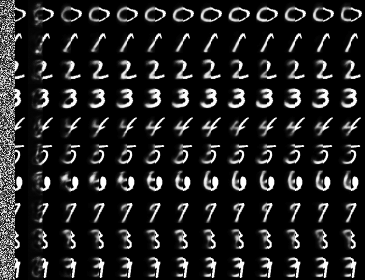}
\vs{5}
\caption{Top: two runs of consecutive samples (one row after the other) generated
from 2-layer GSN model,
showing fast mixing between classes, nice and sharp images. Note: only every fourth sample is shown; see the supplemental material for the samples in between.
Bottom: conditional Markov chain, with the right half of the image clamped to
one of the MNIST digit images and the left half successively resampled, illustrating
the power of the generative model to stochastically fill-in missing inputs. See also Figure~\ref{fig:samples+inpainting} for longer runs.}
\label{fig:samples+inpainting_small}
\vs{5}
\end{figure}

Interestingly, the GSN and the DBN-2 actually perform slightly better than
when using samples directly coming from the MNIST training set, maybe because
they generate more ``prototypical'' samples (we are using mean-field outputs).

Figure~\ref{fig:samples+inpainting_small} shows a single run of consecutive samples
from this trained model (see Figure~\ref{fig:samples+inpainting} for longer runs),
illustrating that it mixes quite well (better
than RBMs) and produces rather sharp digit images. The figure shows
that it can also stochastically complete missing values: the left half
of the image was initialized to random pixels and the right side was clamped
to an MNIST image. The Markov chain explores plausible variations of 
the completion according to the trained conditional distribution.

A smaller set of experiments was also run on TFD, yielding a test set
Parzen log-likelihood bound of 1890 $\pm 29$. The setup is exactly the same
and was not tuned after the MNIST experiments. A DBN-2 yields a 
Parzen log-likelihood bound of 1908 $\pm 66$, which is indistinguishable
statistically, while an RBM yields 604 $\pm$ 15. One out of every 2
consecutive samples from the GSN-3
model are shown in Figure~\ref{fig:tfd-samples_small} (see
Figure~\ref{fig:tfd-samples} for longer runs without skips).

\begin{figure}[htpb]
\vs{3}
\centering
\includegraphics[width=1.0\linewidth]{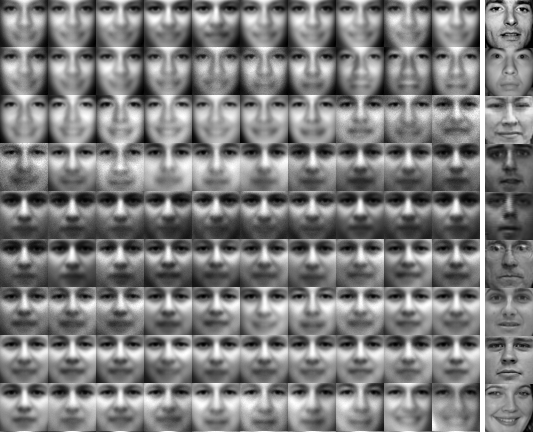}  \vspace*{2mm} %
\vs{8}
\caption{GSN samples from a 3-layer model trained on the TFD dataset. Every second sample is shown; see supplemental material for every sample. At the end of each row, we show the nearest example from the training set to the last sample on that row, to illustrate that the distribution is not merely copying the training set. See also Figure~\ref{fig:tfd-samples} for longer runs without skips.}
\label{fig:tfd-samples_small}
\vs{3}
\end{figure}

\begin{table*}[htpb]
\vs{3}
\caption{Test set log-likelihood lower bound (LL) obtained by a Parzen density estimator
constructed using 10000 generated samples, for different generative models trained
on MNIST.
The LL is not directly comparable to AIS likelihood estimates
because we use a Gaussian mixture rather
than a Bernoulli mixture to compute the likelihood, but we can compare with
\citet{Rifai-icml2012,Bengio-et-al-ICML2013,Bengio-et-al-NIPS2013} (from which we took the last three columns).
A DBN-2 has 2 hidden layers, a CAE-1 has 1 hidden layer, and a CAE-2 has 2.
The DAE is basically a GSN-1, with no injection of noise inside the network.
The last column uses 10000 MNIST training examples to train the Parzen density estimator.
}
\label{tab:LL}
\begin{center}
\begin{small}
\begin{sc}
\begin{tabular}{lrrrrrr} 
\toprule
& GSN-2 & DAE & DBN-2 & CAE-1 & CAE-2  & MNIST\\
\midrule
Log-likelihood lower bound  &  214 & 144 & 138   & 68  & 121  & 24 \\
Standard error              &  1.1 & 1.6 &  2.0  & 2.9 & 1.6  & 1.6\\
\bottomrule
\end{tabular}
\end{sc}
\end{small}
\end{center}
\svs{3}
\end{table*}

\section{Conclusion}
\svs{2}

We have introduced a new approach to training generative models, called
Generative Stochastic Networks (GSN), that is an alternative to maximum
likelihood, with the objective of avoiding the intractable marginalizations
and the danger of poor approximations of these marginalizations. The
training procedure is more similar to function approximation than to
unsupervised learning because the reconstruction distribution is simpler
than the data distribution, often unimodal (provably so in the limit of
very small noise).  This makes it possible to train unsupervised models that
capture the data-generating distribution simply using back-prop and
gradient descent (in a computational graph that includes noise injection).
The proposed theoretical results state that under mild conditions
(in particular that the noise injected in the networks prevents perfect reconstruction),
training the model to denoise and reconstruct
its observations (through a powerful family of reconstruction
distributions) suffices to capture the data-generating distribution through
a simple Markov chain. Another way to put it is that we are training the
transition operator of a Markov chain whose stationary distribution
estimates the data distribution, and it turns out that this is a much
easier learning problem because the normalization constant for this
conditional distribution is generally dominated by fewer modes. These
theoretical results are extended to the case where the corruption is local
but still allows the chain to mix and to the case where some inputs are
missing or constrained (thus allowing to sample from a conditional
distribution on a subset of the observed variables or to learned structured
output models). The GSN framework is shown to lend to dependency networks a
valid estimator of the joint distribution of the observed variables even
when the learned conditionals are not consistent, also allowing to prove
consistency of generalized pseudolikelihood training, associated with the
stationary distribution of the corresponding GSN (that randomly chooses a
subset of variables and then resamples it). Experiments have
been conducted to validate the theory, in the case where the GSN
architecture emulates the Gibbs sampling process of a Deep Boltzmann
Machine, on two datasets. A quantitative evaluation of the samples
confirms that the training procedure works very well (in this case
allowing us to train a deep generative model without layerwise pretraining)
and can be used to perform conditional sampling of a subset of
variables given the rest.

\subsubsection*{Acknowledgements}

The authors would like to acknowledge the stimulating discussions and
help from Vincent Dumoulin,
Pascal Vincent, Yao Li, Aaron Courville, Ian Goodfellow, and Hod Lipson,
as well as funding from NSERC, CIFAR (YB is a CIFAR Senior Fellow), NASA (JY is a Space Technology Research Fellow),
and the Canada Research Chairs.

\appendix

\section{Supplemental Experimental Results}

Experiments evaluating the ability of the GSN models to generate good samples
were performed on the MNIST and TFD datasets, following the setup 
in~Bengio et al. (2013c). 
Theorem 2 requires $H_0$ to have the same distribution as $H_1$ (given $X_0$) during training,
and the main paper suggests a way to achieve this by initializing each training chain with
$H_0$ set to the previous value of $H_1$ when the same example $X_0$ was shown. However, we
did not implement that procedure in the experiments below, so that is left for future work
to explore. 

Networks with 2 and 3 hidden layers
were evaluated and compared to regular denoising auto-encoders (just 1
hidden layer, i.e., the computational graph separates into separate ones
for each reconstruction step in the walkback algorithm). They all have $\tanh$
hidden units and pre- and post-activation Gaussian noise of standard
deviation 2, applied to all hidden layers except the first.
In addition, at each step in the
chain, the input (or the resampled $X_t$) is corrupted with salt-and-pepper
noise of 40\% (i.e., 40\% of the pixels are corrupted, and replaced with a
0 or a 1 with probability 0.5). Training is over 100 to 600 epochs at most, with
good results obtained after around 100 epochs, using stochastic gradient descent
(minibatch size = 1).  Hidden layer sizes vary
between 1000 and 1500 depending on the experiments, and a learning rate of
0.25 and momentum of 0.5 were selected to approximately minimize the
reconstruction negative log-likelihood. The learning rate is reduced
multiplicatively by $0.99$ after each epoch.  Following~Breuleux et al. (2011), 
the quality of the samples
was also estimated quantitatively by measuring the log-likelihood
of the test set under a Parzen density estimator constructed from
10000 consecutively generated samples (using the real-valued mean-field reconstructions
as the training data for the Parzen density estimator). This can be
seen as an {\em lower bound on the true log-likelihood}, with the bound
converging to the true likelihood as we consider more samples and
appropriately set the smoothing parameter of the Parzen 
estimator\footnote{However, in this paper, to be consistent with
the numbers given in Bengio et al. (2013c) 
we used a Gaussian
Parzen density, which (in addition to being lower rather than
upper bounds) makes the numbers not comparable with the
AIS log-likelihood upper bounds for binarized images reported in some papers
for the same data.}.
Results are summarized in Table~\ref{tab:LL}. The test set Parzen log-likelihood bound
was not used to select among model architectures, but visual inspection of
samples generated did guide the preliminary search reported here.
Optimization hyper-parameters (learning rate, momentum, and
learning rate reduction schedule) were selected based on the 
reconstruction log-likelihood training objective. The Parzen log-likelihood bound obtained
with a two-layer model on MNIST is 214 ($\pm$ standard error of 1.1), while the log-likelihood
bound obtained by a single-layer model (regular denoising auto-encoder, DAE in
the table) is
substantially worse, at -152$\pm$2.2.
In comparison,~Bengio et al. (2013c) 
report a log-likelihood bound of -244$\pm$54 for
RBMs and 138$\pm$2 for a 2-hidden layer DBN, using the same setup. We have
also evaluated a 3-hidden layer DBM~(Salakhutdinov \& Hinton, 2009), 
using the weights provided by the author, and obtained a Parzen log-likelihood bound of 32$\pm$2.
See \texttt{http://www.mit.edu/{\textasciitilde}rsalakhu/DBM.html} for details.
Figure~\ref{fig:samples+inpainting} shows two runs of consecutive samples
from this trained model,
illustrating that it mixes quite well (better
than RBMs) and produces rather sharp digit images. The figure shows
that it can also stochastically complete missing values: the left half
of the image was initialized to random pixels and the right side was clamped
to an MNIST image. The Markov chain explores plausible variations of 
the completion according to the trained conditional distribution.

A smaller set of experiments was also run on TFD, yielding for a GSN a test set
Parzen log-likelihood bound of 1890 $\pm 29$. The setup is exactly the same
and was not tuned after the MNIST experiments. A DBN-2 yields a 
Parzen log-likelihood bound of 1908 $\pm 66$, which is undistinguishable
statistically, while an RBM yields 604 $\pm$ 15. A run of
consecutive samples from the GSN-3
model are shown in Figure~\ref{fig:tfd-samples}.
Figure~\ref{fig:early-samples} shows consecutive samples obtained early on during
training, after only 5 and 25 epochs respectively, illustrating the fast
convergence of the training procedure.

\begin{figure*}[ht]
\centering
\includegraphics[width=\textwidth]{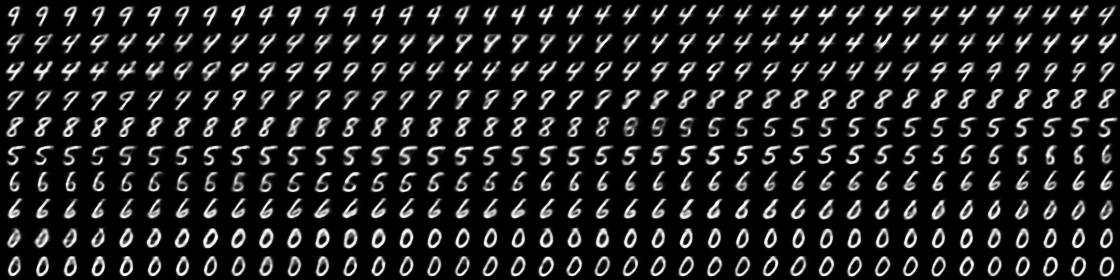} 

\includegraphics[width=\textwidth]{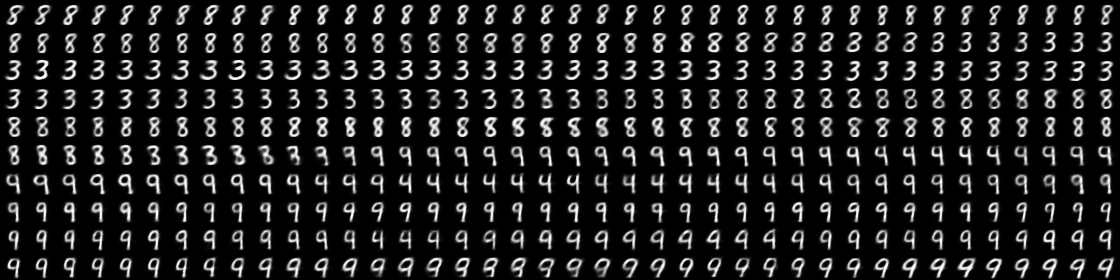}

\vspace*{1.5mm}
\includegraphics[width=\textwidth]{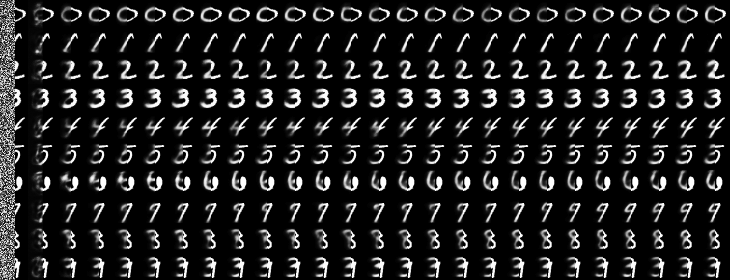}
\caption{These are expanded plots of those in Figure~\ref{fig:samples+inpainting_small}.
{\em Top:} two runs of consecutive samples (one row after the other) generated
from a 2-layer GSN model,
showing that it mixes well between classes and produces nice and sharp images. Figure~\ref{fig:samples+inpainting_small} contained only one in every four samples, whereas here we show every sample.
{\em Bottom:} conditional Markov chain, with the right half of the image clamped to
one of the MNIST digit images and the left half successively resampled, illustrating
the power of the trained generative model to stochastically fill-in missing inputs.
Figure~\ref{fig:samples+inpainting_small} showed only 13 samples in each chain; here we show 26.}
\label{fig:samples+inpainting}
\end{figure*}

\begin{figure*}[ht]
\centering
\includegraphics[width=0.49\textwidth]{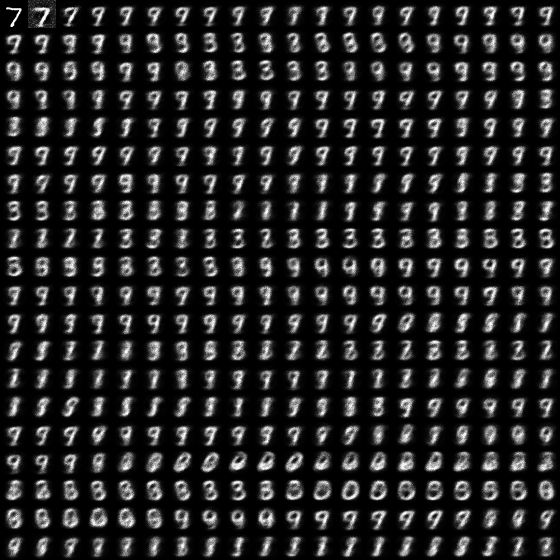}  \vspace*{2mm} %
\includegraphics[width=0.49\textwidth]{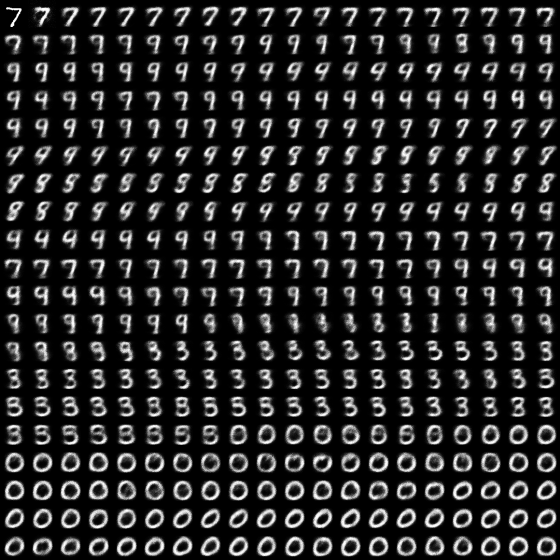}
\caption{Left: consecutive GSN samples obtained after 10 training epochs. Right: GSN
  samples obtained after 25 training epochs. This shows quick convergence to a model that samples well. The samples in
  Figure~\ref{fig:samples+inpainting} are obtained after 600 training
  epochs.}
\label{fig:early-samples}
\end{figure*}

\begin{figure*}[ht]
\centering
\includegraphics[width=1\textwidth]{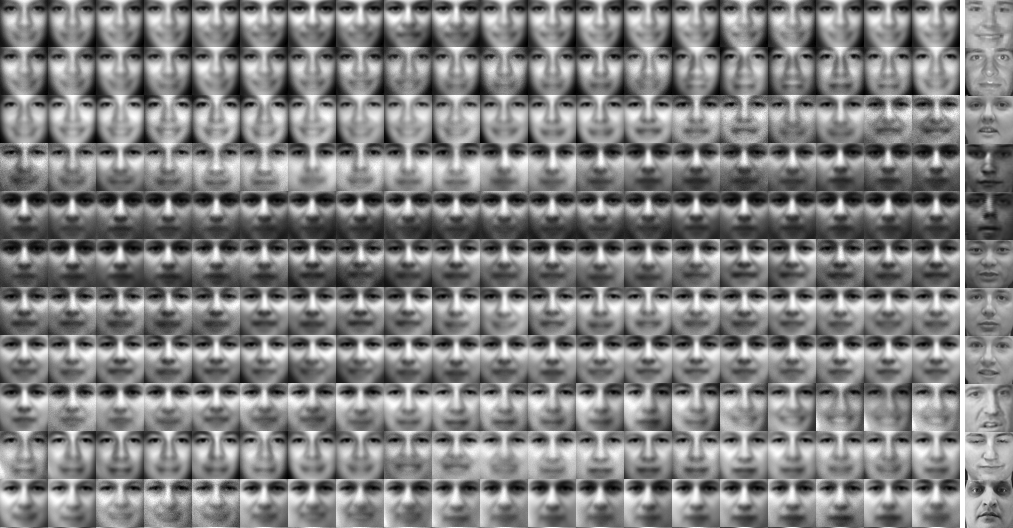}  \vspace*{2mm} %
\caption{Consecutive GSN samples from a 3-layer model trained on the TFD dataset.  At the end of each row, we show the nearest example from the training set to the last sample on that row to illustrate that the distribution is not merely copying the training set.}
\label{fig:tfd-samples}
\end{figure*}

\bibliography{strings,strings-shorter,ml,aigaion-shorter,cultrefs}
\bibliographystyle{icml2014}
\end{document}